\documentclass[conference]{IEEEtran}
\IEEEoverridecommandlockouts
\usepackage{cite}
\usepackage{amsthm}
\usepackage{amsmath,amssymb,amsfonts}
\usepackage{graphicx}
\usepackage{textcomp}
\usepackage{xcolor}
\def\BibTeX{{\rm B\kern-.05em{\sc i\kern-.025em b}\kern-.08em
    T\kern-.1667em\lower.7ex\hbox{E}\kern-.125emX}}

\usepackage{algorithm}
\usepackage{algpseudocode}
\usepackage[utf8]{inputenc} 
\usepackage[T1]{fontenc}    
\usepackage{hyperref}       
\usepackage{url}            
\usepackage{booktabs}       
\usepackage{amsfonts}       
\usepackage{nicefrac}       
\usepackage{microtype}      

\usepackage{adjustbox}
\usepackage{graphicx}
\usepackage{subfigure}
\usepackage{booktabs} 
\usepackage{amsmath,amssymb,amsfonts}
\usepackage{graphicx}
\usepackage{textcomp}
\usepackage{xcolor}
\usepackage{booktabs}
\usepackage{multirow}
\usepackage{hyperref}
\usepackage{cleveref}
\usepackage{caption}
\usepackage{times}
\usepackage{fancyhdr}
\usepackage{hyperref}

\def\comp{\raise 1pt \hbox{$\scriptstyle\circ$}}

\def\argmin{\mathop{\rm argmin}\limits}

\def\minimize{\mathop{\rm min}\limits}

\def\st{\mathop{\rm subject\ to}}

\def\upto{{\raise 1pt \hbox{$\scriptstyle \,\nearrow\,$}}}
\def\downto{{\raise 1pt \hbox{$\scriptstyle \,\searrow\,$}}}

\begin{document}

\newtheorem{theorem}{Theorem}
\newtheorem{definition}{Definition}

\newcommand{\norm}[1]{\left\lVert#1\right\rVert}
\newcommand{\boldx}{\mbox{$\mathbf{x}$}}
\newcommand{\boldy}{\mbox{$\mathbf{y}$}}
\newcommand{\boldf}{\mbox{$\mathbf{f}$}}
\newcommand{\boldz}{\mbox{$\mathbf{z}$}}
\newcommand{\boldF}{\mbox{$\mathbf{F}$}}
\newcommand{\boldG}{\mbox{$\mathbf{G}$}}
\newcommand{\boldg}{\mbox{$\mathbf{g}$}}
\newcommand{\boldh}{\mbox{$\mathbf{h}$}}
\newcommand{\boldH}{\mbox{$\mathbf{H}$}}
\newcommand{\boldzero}{\mbox{$\mathbf{0}$}}
\newcommand{\Rbb}{\mbox{$\mathbb R$}}

\newcommand{\boldq}{\mbox{$\mathbf{q}$}}
\newcommand{\boldtau}{\mbox{\boldmath$\tau$}}

\title{A Linear Programming Enhanced Genetic Algorithm for Hyperparameter Tuning in Machine Learning}

\author{\IEEEauthorblockN{Ankur Sinha, IEEE Senior Member}
\vspace{-3mm}
\IEEEauthorblockA{\textit{} \\
\textit{Brij Disa Centre for Data Science and Artificial Intelligence}\\
Indian Institute of Management Ahmedabad\\
Gujarat, India 380015 \\
asinha@iima.ac.in}
\and
\IEEEauthorblockN{Paritosh Pankaj}
\vspace{-3mm}
\IEEEauthorblockA{\textit{} \\
\textit{Department of Statistics and Data Science}\\
Indian Institute of Technology Kanpur\\
Uttar Pradesh, India 208016\\
ppankaj21@iitk.ac.in}
}

\IEEEoverridecommandlockouts

\maketitle

\IEEEpubidadjcol

\begin{abstract}
In this paper, we formulate the hyperparameter tuning problem in machine learning as a bilevel program. The bilevel program is solved using a micro genetic algorithm that is enhanced with a linear program. While the genetic algorithm searches over discrete hyperparameters, the linear program enhancement allows hyper local search over continuous hyperparameters. The major contribution in this paper is the formulation of a linear program that supports fast search over continuous hyperparameters, and can be integrated with any hyperparameter search technique. It can also be applied directly on any trained machine learning or deep learning model for the purpose of fine-tuning. We test the performance of the proposed approach on two datasets, MNIST and CIFAR-10. Our results clearly demonstrate that using the linear program enhancement offers significant promise when incorporated with any population-based approach for hyperparameter tuning.
\end{abstract}


\section{Introduction}\label{introduction}
Hyperparameter optimization is an incredibly challenging task in machine learning, as hyperparameters are external to the model and can't be determined based on the training data alone. These common hyperparameters include, network architecture (for example, number of layers and number of neurons per layer), optimization parameters (for example, learning rate and momentum), and regularization parameters (for example, weight decay and dropout). The most common approach to identify the right set of hyperparameters involves training models with different hyperparameters on the training dataset and then evaluating the models on the validation dataset. The best performing hyperparameters are chosen.


The hyperparameter optimization problem is intrinsically a bilevel optimization task where the upper level problem searches for the optimal hyperparameters and the lower level problem searches for the optimal model parameters for the corresponding hyperparameters. In the context of evolutionary algorithms as well a number of algorithm hyperparameters have to be tuned and their optimal choice can be made using a bilevel optimization approach.
Formulating the hyperparameter optimization problem as a bilevel optimization task is a familiar approach in machine learning \cite{bennett08} and also in evolutionary computation \cite{gecco14}.

A bilevel optimization problem involves two levels of optimization with each level having its own objective function, set of variables, and set of constraints. A large body of literature exists on bilevel optimization for which the readers may refer to \cite{sinha2017review,dempe2002foundations,bard2013practical}. A bilevel optimization problem is challenging because the upper level variables appear as parameters in the lower level optimization problem, while the lower level problem has to be optimized with respect to the lower level variables. Solving the lower level optimization problem for a given set of upper level variables and ensuring that the upper level constraints are satisfied lead to a feasible solution to the bilevel optimization problem. Linear bilevel programs \cite{FoMc81,WeHs91,TuMiVa93} and quadratic bilevel problems \cite{EdBa91,AlHoPa92} are widely solved using Karush-Kuhn-Tucker based approach. Researchers have looked at other approaches based on gradients \cite{savard1994steepest}, penalty \cite{WhAn93,my-swevo18,kleinert2021computing}, trust-region \cite{marcotte2001trust,colson2005trust}, among others, to solve bilevel optimization problems. In the domain of evolutionary algorithms, there are a number of nested approaches \cite{my-caor14,angelo13,islam2017surrogate} that are used to handle bilevel optimization problems; however, these methods involve solving a large number of lower level optimization problems. Recently, the algorithm development on evolutionary bilevel optimization has focused on exploiting the lower level reaction set mapping and the lower-level optimal value function mapping \cite{sinha2021solving,my-joh20,my-ejor17,angelo2014differential}. To begin with, we will formulate the hyperparameter optimization problem as a bilevel problem and will stick to the nomenclature commonly used by the machine learning community. Let the upper level variables (hyperparameters) be denoted by $\lambda$ and the lower level variables (model parameters) be denoted by $w$. If the upper level objective (validation loss) is given as $F(\lambda,w)$ and the lower level objective (training loss) is given as $f(\lambda,w)$, then the hyperparameter optimization problem is defined as follows:
\begin{align}\label{mod:originalBilevel}
\begin{split}
\minimize_{\lambda,w} \quad & F(\lambda,w; S^V) \\
\st\quad  & \\
 & \hspace{-12mm} w \in \argmin_w \{f(\lambda,w; S^{T})\}
 \end{split}
\end{align}
where $S_{T}$ represents the set of training examples and $S_{V}$ represents the set of validation examples. Under general cases, the upper level and the lower level problems may contain constraints as well. 

In this paper, we will consider two types of hyperparameters to demonstrate our ideas. The first set of hyperparameters, denoted by $\lambda_{d}$, would include discrete hyperparameters for which we have chosen the network architecture parameters (number of layers and number of neurons). The second set of hyperparameters, denoted by $\lambda_{c}$, would include continuous hyperparameters for which we have chosen regularization hyperparameters (weight decay). In our proposed approach, we will handle the discrete and continuous parameters using a micro genetic algorithm (micro-GA), and perform a hyper local search for the continuous hyperparameters using a linear programming approach. The main contribution in this paper is the design of the linear program that would help us perform hyper local search within the micro-GA.

A wide range of techniques exists to address hyperparameter optimization problem defined in \eqref{mod:originalBilevel}. The
popular strategies include naive methods, like grid search and random search, where a number of hyperparameter vectors are sampled from the hyperparameter space and models are optimized on the training dataset for each of the sample hyperparameter vector. The models are then evaluated on the validation dataset and the hyperparameter vector that leads to the least validation loss is chosen. Bergstra et al. \cite{bergstra2011algorithms} demonstrated that the random search surpasses grid search in terms of computational performance, thus making it preferable. Hyperband \cite{li2017hyperband}, which is an extension of random search, intelligently allocates computational resources to promising configurations via a multi-armed bandit technique while searching for the best hyperparameters. Bayesian optimization happens to be the gold standard for hyperparameter optimization \cite{hutter2011sequential,bergstra2011algorithms,snoek2012practical,snoek2015scalable}. In these methods, a probabilistic model is created for the validation objective based on which an informed decision is made for sampling the next hyperparameter vector. A common requirement in most of these approaches is to strike the right balance between exploration and exploitation, as heavy exploration tends to be computationally very costly in hyperparameter optimization. Common methods used for probabilistic estimation of the objective function includes, tree Parzen estimator \cite{bergstra2011algorithms}, Gaussian process estimation \cite{snoek2012practical}, and sequential model-based approach \cite{hutter2011sequential}. Most of these methods, model-free and model-based, suffer from the curse of dimensionality and perform poorly with increase in number of hyperparameters \cite{maclaurin2015gradient}. These approaches are also referred to as black-box approaches as they do not utilize the underlying structure of the hyperparameter optimization problem. Response surface estimation-based \cite{sinha2020gradient,mackay2019self,franceschi2018bilevel,lorraine2018stochastic,maclaurin2015gradient} and hypergradient-based \cite{pedregosa2016hyperparameter,bengio2000gradient} optimization approaches have also been popular lately in the context of hyperparameter optimization. Some aspects of our study fall within the domain of gradient estimation for the bilevel hyperparameter optimization problem. Further references on hyperparameter optimization can be found in \cite{feurer2019hyperparameter}.




The paper is organized as follows. To begin with, we propose the model fine-tuning approach with respect to continuous hyperparameters in Section~\ref{sec:localsearch}, followed by a detailed discussion on its integration in micro-GA in the form of hyper local search in Section~\ref{propmethod}. 
In Section~\ref{sec:experiments} we provide extensive results on two datasets to demonstrate the effectiveness of our approach. Finally, we conclude the paper in Section~\ref{sec:conclusions}. The central notations used in our study are summarized in Table \ref{tab:centralNotation}.

\begin{table*}[]
\centering
\caption{Central Notation}
\label{tab:centralNotation}
\resizebox{\textwidth}{!}{%
\begin{tabular}{@{}lll@{}}
\toprule
\textbf{Category}                   & \textbf{Notation} & \multicolumn{1}{c}{\textbf{Description}} \\ \midrule
\multirow{4}{*}{Dataset}           & \multirow{2}{*}{$S_{T}$}       &   \begin{tabular}[c]{@{}l@{}}$S_{T}= \{(x_i,y_i)\}_{i = 1}^{N^T}$; training set, where $x$ and $y$ are combination of \\input features and output classes, and $N^T$ is the number of training examples\end{tabular}\\\cmidrule{2-3}
                                    &  \multirow{2}{*}{$S_{V}$}        & \begin{tabular}[c]{@{}l@{}} $S_{V}= \{(x_i,y_i)\}_{i = 1}^{N^V}$; validation set, where $x$ and $y$ are combination of \\input features and output classes, and $N^V$ is the number of validation examples\end{tabular}\\ \cmidrule{1-3}
\multirow{2}{*}{Bilevel variables} & $\lambda = (\lambda_d,\lambda_c)$          & discrete and continuous hyperparameters (upper level variables) \\\cmidrule{2-3}
                                    & $ w$               & model parameters (lower level variables)                      \\ \cmidrule{1-3}
\multirow{2}{*}{Objectives}         & $F(\lambda,w; S^V)$                 & upper level objective function                  \\ \cmidrule{2-3}
                                    & $f(\lambda,w; S^T)$                 & lower level objective function                      \\\cmidrule{1-3}
Loss function                       & $l$                 &  training loss $l(w; S^T)$ and validation loss $l(w; S^V)$                \\ \cmidrule{1-3}
Regularization                      & $\Theta$                  &  regularization function (L$_2$ regularization is used in this paper)                 \\\cmidrule{1-3}
\multirow{2}{*}{Direction vectors} & $d_{\lambda_c}$          & the descent (with respect to $F$) direction vector for continuous hyperparameters \\\cmidrule{2-3}
                                    & $d_{w}$               & the descent (with respect to $F$) direction vector for model parameters                    \\
 \bottomrule
\end{tabular}
}
\end{table*}

\section{Fine-tuning Machine Learning Models}\label{sec:localsearch}
In this section, we discuss a model fine-tuning approach that can be applied on any machine learning model that has been learned on a given training dataset. The approach works by refining the chosen continuous hyperparameters and the optimized model parameters in its vicinity using a linear programming approach. Let us focus only on continuous hyperparameters that can be varied and let us assume the discrete hyperparameters to be fixed. In this case, the bilevel optimization problem for optimizing the continuous hyperparameters can be written as follows:
\begin{align}\label{mod:continuousHyp}
\begin{split}
\minimize_{\lambda_c,w} \quad & F(\lambda_c,w; S^V) \\
\st\quad  & \\
 & \hspace{-12mm} w \in \argmin_w \{f(\lambda_c,w; S^{T})\}
 \end{split}
\end{align}
Let us say that for a given value of the continuous hyperparameter, $\lambda_{c}^{\circ}$, we optimize the model parameters, i.e. solve the lower level problem, and obtain $w^{\circ}$. We wish to fine-tuning the model $M(\lambda_{c}^{\circ},w^{\circ})$, by moving in a direction, $(d_{\lambda_c},d_w)$, such that we obtain a new model $M(\lambda_{c}^{\circ}+td_{\lambda_c},w^{\circ}+td_w)$ (for some non-negative value of $t$) that provides a better upper level function value and $w^{\circ}+td_w$ remains optimal for $\lambda_{c}^{\circ}+td_{\lambda_c}$. It would be ideal to choose the direction in such a way that it leads to the steepest descent for the upper level objective function while satisfying the lower level optimality conditions. Such a direction is nothing but the negative of the gradient of the bilevel optimization problem \eqref{mod:continuousHyp}. We will next attempt to derive this direction of steepest descent.

The assumptions for the derivation are that the upper level function $F(\lambda_c,w; S^V)$ is at least once differentiable and the lower level function $f(\lambda_c,w; S^{T})$ is at least twice differentiable. We also assume that for any given value of $\lambda_c$, there always exists a solution $w \in \argmin_w \{f(\lambda_c,w; S^{T})\}$.

\begin{theorem}
At a given point $(\lambda_{c}^{\circ},w^{\circ})$, such that, $w^{\circ} \in \argmin_w \{f(\lambda_{c}^{\circ},w; S^{T})$ the steepest descent direction for~\eqref{mod:continuousHyp} can be obtained by solving the following problem:
\begin{align}\label{mod:steepestDescent1}
\begin{split}
\minimize_{d_{\lambda_c},d_w} \quad & \nabla_{\lambda_c} F(\lambda^{\circ}_{c},w^{\circ}; S^V)^{T} d_{\lambda_c} + \nabla_{w} F(\lambda^{\circ}_{c},w^{\circ}; S^V)^{T} d_{w}\\
\st  & \\
 & \hspace{-12mm} d_w \in \argmin_{d_w} \Biggl\{\left[ {\begin{array}{cc}
    d_{\lambda_{c}} \\
     d_{w} \\
  \end{array} } \right]^T \nabla_{(\lambda_c,w)}^{2} f(\lambda^{\circ}_{c},w^{\circ}; S^{T}) \left[ {\begin{array}{cc}
    d_{\lambda_{c}} \\
     d_{w} \\
  \end{array} } \right]\Biggr\}\\
& \hspace{-12mm} -1 \le d_{\lambda_c} \le 1
 \end{split}
\end{align}
\end{theorem}

\begin{proof}
For a given direction vector $(d_{\lambda_c},d_w)$ and gradient of the upper level function $(\nabla_{\lambda_c} F(\lambda^{\circ}_{c},w^{\circ}; S^V), \nabla_{w} F(\lambda^{\circ}_{c},w^{\circ}; S^V))$ at the point $(\lambda^{\circ}_{c},w^{\circ})$, clearly $(d_{\lambda_c},d_w)$ represents the upper level descent direction if the dot product $\nabla_{\lambda_c} F(\lambda^{\circ}_{c},w^{\circ}; S^V)^{T} d_{\lambda_c} + \nabla_{w} F(\lambda^{\circ}_{c},w^{\circ}; S^V)^{T} d_{w} < 0$.
Also if $D^{\circ}$ is the acceptable set of direction vectors at $(\lambda^{\circ}_{c},w^{\circ})$, then the following would lead to the steepest descent direction for \eqref{mod:continuousHyp} at the point $(\lambda^{\circ}_{c},w^{\circ})$.
\begin{align*}
\argmin_{d_{\lambda_c},d_w} \{ \nabla_{\lambda_c} F(\lambda^{\circ}_{c},w^{\circ}; S^V)^{T} d_{\lambda_c} + \nabla_{w} F(\lambda^{\circ}_{c},w^{\circ}; S^V)^{T} d_{w}\\
: (d_{\lambda_c},d_w) \in D^{\circ} \}\\
\end{align*}
We know that $w^{\circ} \in \argmin_w \{f(\lambda_c^{\circ},w; S^{T})\}$, therefore, $\nabla_w f(\lambda_c^{\circ},w^{\circ}; S^{T}) = 0$. If $\lambda_c$ changes infinitesimally, as $\lim_{t\to0} \lambda^{\circ}_c + t d_{\lambda_c}$, we would like to know $\lim_{t\to0}  w^{\circ} + t d_{w}$, such that, 
\begin{align}
d_w \in \argmin_{d_w} \{f(\lambda_c^{\circ}+t d_{\lambda_c},w^{\circ} + t d_{w}; S^{T})\} \label{eq:llmin}
\end{align}
We made the assumption that the lower level problem always has an optimal solution for any given upper level vector, so an optimal $d_w$ exists for a given value of $t$ and $d_{\lambda_c}$. Essentially, at $(\lambda^{\circ}_{c},w^{\circ})$ when the upper level vector changes along the direction $d_{\lambda_c}$, we want to know the direction $d_{w}$ along which the lower level vector should change so that it remains optimal for the lower level problem.

We have assumed $f(\lambda_c^{\circ},w^{\circ})$ to be twice differentiable, therefore, we can write its Taylor's expansion around $(\lambda_c^{\circ},w^{\circ})$ with second-order approximation as follows (dropping $S^T$ for brevity):
\begin{align*}
& f(\lambda_c^{\circ}+t d_{\lambda_c},w^{\circ}+t d_{w}) = f(\lambda_c^{\circ},w^{\circ}) + \\ 
& \nabla_{\lambda_c} f(\lambda_c^{\circ},w^{\circ})^T td_{\lambda_c} + \nabla_{w} f(\lambda_c^{\circ},w^{\circ})^T td_{w} + \\
& \frac{t^2}{2}\left[ {\begin{array}{cc}
    d_{\lambda_{c}} \\
     d_{w} \\
  \end{array} } \right]^T \nabla_{(\lambda_c,w)}^{2} f(\lambda^{\circ}_{c},w^{\circ}) \left[ {\begin{array}{cc}
    d_{\lambda_{c}} \\
     d_{w} \\
  \end{array} } \right]
\end{align*}
The above expansion has been written as 4 terms, where the second and third terms are first-order terms written in two parts. Since $\nabla_w f(\lambda_c^{\circ},w^{\circ}) = 0$, the third term can be ignored. Therefore, we get the following:
\begin{align*}
   & \min_{d_w} f(\lambda_c^{\circ}+t d_{\lambda_c},w^{\circ} + t d_{w}; S^{T}) = \\ 
   & f(\lambda_c^{\circ},w^{\circ}) + 
    \nabla_{\lambda_c} f(\lambda_c^{\circ},w^{\circ})^T td_{\lambda_c} + \\
& \frac{t^2}{2} \min_{d_w} \left[ {\begin{array}{cc}
    d_{\lambda_{c}} \\
     d_{w} \\
  \end{array} } \right]^T \nabla_{(\lambda_c,w)}^{2} f(\lambda^{\circ}_{c},w^{\circ}) \left[ {\begin{array}{cc}
    d_{\lambda_{c}} \\
     d_{w} \\
  \end{array} } \right]
\end{align*}
which implies,
\begin{align*}
   & \argmin_{d_w} \Biggl\{ f(\lambda_c^{\circ}+t d_{\lambda_c},w^{\circ} + t d_{w}; S^{T}) \Biggr\} = \\ 
& \argmin_{d_w} \Biggl\{ \left[ {\begin{array}{cc}
    d_{\lambda_{c}} \\
     d_{w} \\
  \end{array} } \right]^T \nabla_{(\lambda_c,w)}^{2} f(\lambda^{\circ}_{c},w^{\circ}) \left[ {\begin{array}{cc}
    d_{\lambda_{c}} \\
     d_{w} \\
  \end{array} } \right] \Biggr\}
\end{align*}
Therefore, \eqref{eq:llmin} can be written as follows:
\begin{align}
d_w \in \argmin_{d_w} \Biggl\{ \left[ {\begin{array}{cc}
    d_{\lambda_{c}} \\
     d_{w} \\
  \end{array} } \right]^T \nabla_{(\lambda_c,w)}^{2} f(\lambda^{\circ}_{c},w^{\circ}) \left[ {\begin{array}{cc}
    d_{\lambda_{c}} \\
     d_{w} \\
  \end{array} } \right] \Biggr\} \label{eq:llmin2}
\end{align}
solving which gives us an optimal $d_w$ for a given $d_{\lambda_c}$. We want that $(d_{\lambda_c},d_w)$ pair that leads to the steepest descent direction while ensuring $d_w$ optimality for any $d_{\lambda_c}$, which we get by solving \eqref{mod:steepestDescent1}. Note that we additionally have $-1 \le d_{\lambda_c} \le 1$ as a constraint at the upper level which restricts the magnitude of the vector otherwise \eqref{mod:steepestDescent1} will be unbounded. This completes the proof of the theorem.
\end{proof}

Interestingly, the same results can also be arrived at as a special case of the results discussed in \cite{savard1994steepest}. 
Next, we attempt to simplify the results further. Given that the lower level problem is an unconstrained optimization problem, we can write the first-order optimality conditions for the lower level problem in \eqref{mod:steepestDescent1}. Let the symmetric matrix $\nabla_{(\lambda_c,w)}^{2} f(\lambda^{\circ}_{c},w^{\circ}; S^{T})$ be denoted as:
$$
\Bigl[h_{ij}\Bigr]_{i=1,j=1}^{p+q,p+q} = \nabla_{(\lambda_c,w)}^{2} f(\lambda^{\circ}_{c},w^{\circ}; S^{T}),
$$
where $p$ and $q$ denote the dimensions of $\lambda_c$ (or $d_{\lambda_c}$) and $w$ (or $d_w$), respectively. Then, the first order conditions for the lower level problem in \eqref{mod:steepestDescent1} can be written as follows:
$$
\Bigl[h_{ij}\Bigr]_{i=p+1,j=1}^{p+q,p+q} \left[ {\begin{array}{cc}
    d_{\lambda_{c}} \\
     d_{w} \\
  \end{array} } \right] = 0
$$
This reduces formulation \eqref{mod:steepestDescent1} into a linear program solving which provides us the steepest descent direction for \eqref{mod:continuousHyp}:
\begin{align}\label{mod:steepestDescentLP}
\begin{split}
\minimize_{d_{\lambda_c},d_w} \quad & \nabla_{\lambda_c} F(\lambda^{\circ}_{c},w^{\circ}; S^V)^{T} d_{\lambda_c} + \nabla_{w} F(\lambda^{\circ}_{c},w^{\circ}; S^V)^{T} d_{w}\\
\st  & \\
 &  \Bigl[h_{ij}\Bigr]_{i=p+1,j=1}^{p+q,p+q} \left[ {\begin{array}{cc}
    d_{\lambda_{c}} \\
    d_{w} \\
  \end{array} } \right] = 0\\
 & -1 \le d_{\lambda_c} \le 1
 \end{split}
\end{align}
We relax the equality constraints in the linear program into inequalities by choosing a small value $\delta$, which leads to the following program:
\begin{align}\label{mod:steepestDescentLP2}
\begin{split}
\minimize_{d_{\lambda_c},d_w} \quad & \nabla_{\lambda_c} F(\lambda^{\circ}_{c},w^{\circ}; S^V)^{T} d_{\lambda_c} + \nabla_{w} F(\lambda^{\circ}_{c},w^{\circ}; S^V)^{T} d_{w}\\
\st  & \\
 &  -\delta \le \Bigl[h_{ij}\Bigr]_{i=p+1,j=1}^{p+q,p+q} \left[ {\begin{array}{cc}
    d_{\lambda_{c}} \\
    d_{w} \\
  \end{array} } \right] \le \delta\\
 & -1 \le d_{\lambda_c} \le 1
 \end{split}
\end{align}
Let the optimal solution to the above problem be denoted as $d_{\lambda_{c}}^{\ast},d_{w}^{\ast}$, then new models along the descent direction can be generated as follows:
\begin{align}
M(\lambda_{c}^{\circ}+td_{\lambda_c}^{\ast},w^{\circ}+td_w^{\ast}): t>0
\end{align}
One may choose a model in the vicinity of $M(\lambda_{c}^{\circ},w^{\circ})$ for a particular value of $t$ (say $t^{\ast}$), such that the validation loss for $M(\lambda_{c}^{\circ}+t^{\ast}d_{\lambda_c}^{\ast},w^{\circ}+t^{\ast}d_w^{\ast})$ is smaller than the validation loss for $M(\lambda_{c}^{\circ},w^{\circ})$.

For the experimentation in this section, the upper and the lower level objective functions in \eqref{mod:steepestDescent1} have been chosen as follows:
\begin{align}
    F(w) & = l(w; S^V)\\
    f(\lambda_c, w) & = l(w; S^{T}) + \Theta(w,\lambda_c)
\end{align}
where $l$ is the average cross-entropy loss function, and $\Theta$ is the L$_2$-regularization function. Note that with these choice of functions, the upper level objective is defined only with respect to $w$ and does not directly involve $\lambda_c$. The optimal model parameter vector is a function of the hyperparameter vector, therefore, one often denotes the optimal model parameters as $w(\lambda_c)$, due to which $F$ has an indirect dependency on $\lambda_c$. With a single regularization hyperparameter we have, $\Theta(w,\lambda_c) =  \lambda_c \sum_{i=1}^{q} w_{i}^2$. Such a regularization approach is also known as weight decay as it promotes the model parameters to be smaller in magnitude, thus preventing overfitting on the training examples. L$_2$-regularization can be extended with additional terms when $\lambda_c$ is a vector, with each term in $\lambda_c$ penalizing a different set of sum of squared weights.

The model $M(\lambda_{c}^{\circ},w^{\circ})$ represents a feasible solution to the bilevel optimization problem in \eqref{mod:steepestDescent1} for which we identify the steepest descent direction by solving the linear program \eqref{mod:steepestDescentLP2}. Let the solution of the linear program be denoted as $(d_{\lambda_{c}}^{\ast},d_{w}^{\ast})$ then the new models can be generated along this direction as follows:
\begin{align*}
  \left[ {\begin{array}{cc}
    \lambda^{n}_c \\
     w^{n} \\
  \end{array} } \right] = \left[ {\begin{array}{cc}
    \lambda^{_\circ}_c \\
     w^{_\circ} \\
  \end{array} } \right] + t\left[ {\begin{array}{cc}
    d_{\lambda_c}^{\ast} \\
     d_w^{\ast} \\
  \end{array} } \right]
\end{align*}
If $M(\lambda_{c}^{\circ}+td_{\lambda_c}^{\ast},w^{\circ}+td_w^{\ast})$ denotes various models along the steepest descent direction, then let $M(\lambda_{c}^{\circ}+t^{\ast}d_{\lambda_c}^{\ast},w^{\circ}+t^{\ast}d_w^{\ast})$ be the model that minimizes the validation loss $l(w; S^V)$. We refer to this exercise as a fine-tuning exercise that will be later incorporated into a genetic algorithm for the purpose of hyper local search. Figure~\ref{fig:explain} shows the fine-tuning exercise graphically and provides a visual representation on how validation and training loss may change along this direction. While validation loss is expected to improve along the descent direction, not much can be inferred about the training loss. 
\begin{center} 
  \includegraphics[scale=0.34]{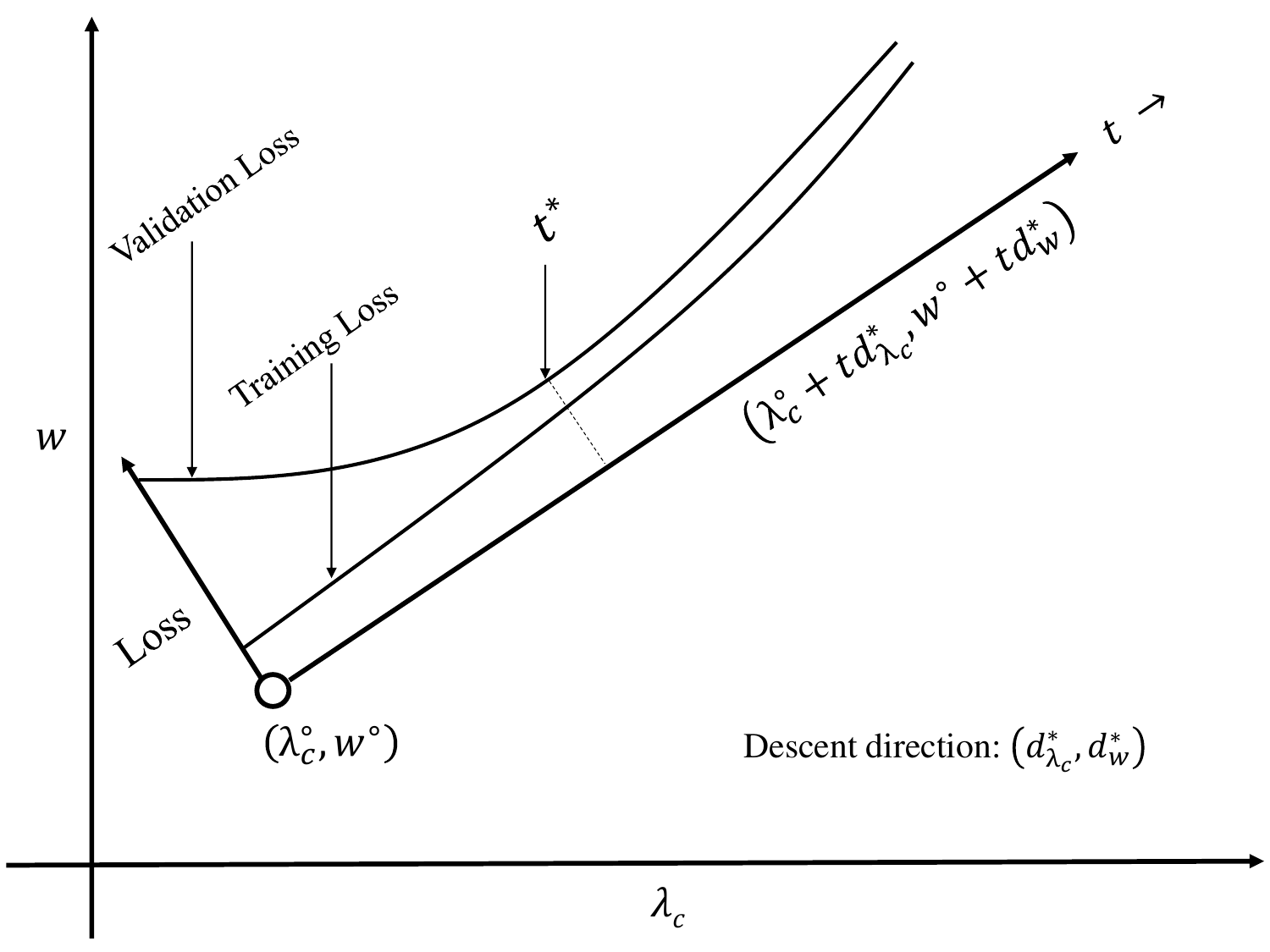}
  \captionof{figure}{$(\lambda_c,w)$ space with the descent direction $(d_{\lambda_{c}}^{\ast},d_{w}^{\ast})$ at $(\lambda_{c}^{\circ},w^{\circ})$. The training and validation loss are also shown along the descent direction.}
  \label{fig:explain}
\end{center}

Next, we present some results for demonstrating the effectiveness of the approach on MNIST dataset \cite{lecun2010mnist} in the context of multi-layer perceptron (MLP) architecture. In MNIST dataset the objective is to solve a multi-classification problem for which we create an MLP model with 5 hidden layers and 50 neurons in each layer. We randomly sample 5,000 data points for training, 2,500 data points for validation and 10,000 data points for testing. The reason for choosing fewer samples for training is to allow overfitting to happen when we create our first model $M(\lambda_{c}^{\circ},w^{\circ})$ without any regularization, i.e. $\lambda_{c}^{\circ}=0$, using stochastic gradient descent. Thereafter, we consider three cases:
\begin{enumerate}
    \item MNIST (1HP):  Solve the linear program in \eqref{mod:steepestDescentLP2} with 1 regularization hyperparameter, i.e. a single regularization hyperparameter for the hidden layers and the output layer
    \item MNIST (2HP): Solve the linear program in \eqref{mod:steepestDescentLP2} with 2 regularization hyperparameters, i.e. 1 regularization hyperparameter for the hidden layers and 1 regularization hyperparameter for the output layer
    \item MNIST (6HP): Solve the linear program in \eqref{mod:steepestDescentLP2} with 6 regularization hyperparameters, i.e. 5 regularization hyperparameters for the hidden layers and 1 regularization hyperparameter for the output layer
\end{enumerate}

\begin{figure}[ht]
\begin{center} 
  \includegraphics[width=.40\textwidth]{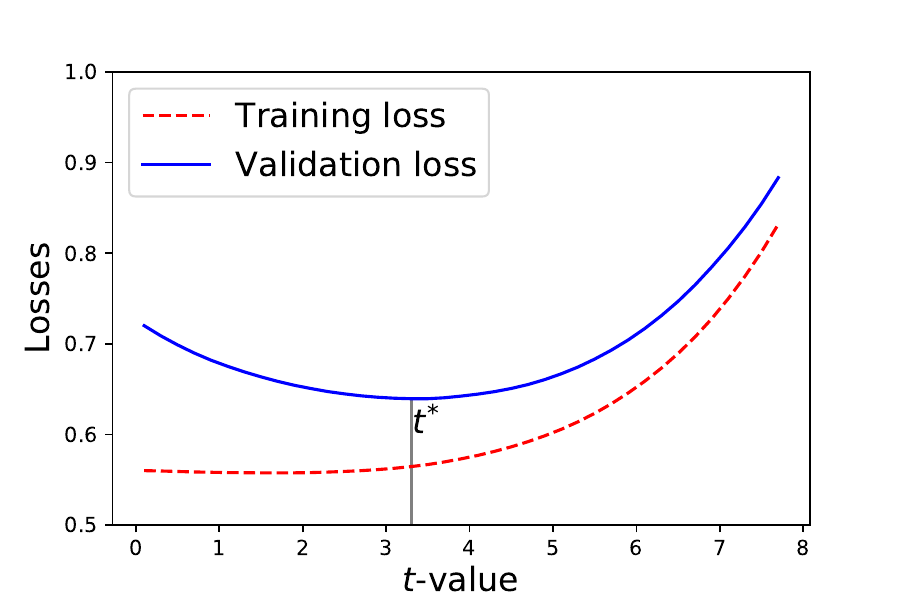}
  \captionof{figure}{Training and validation losses while moving along the steepest descent direction for MNIST (1HP).}
  \label{fig:hyper1}
\end{center}
\end{figure}
\begin{figure}[ht]
\begin{center} 
  \includegraphics[width=.40\textwidth]{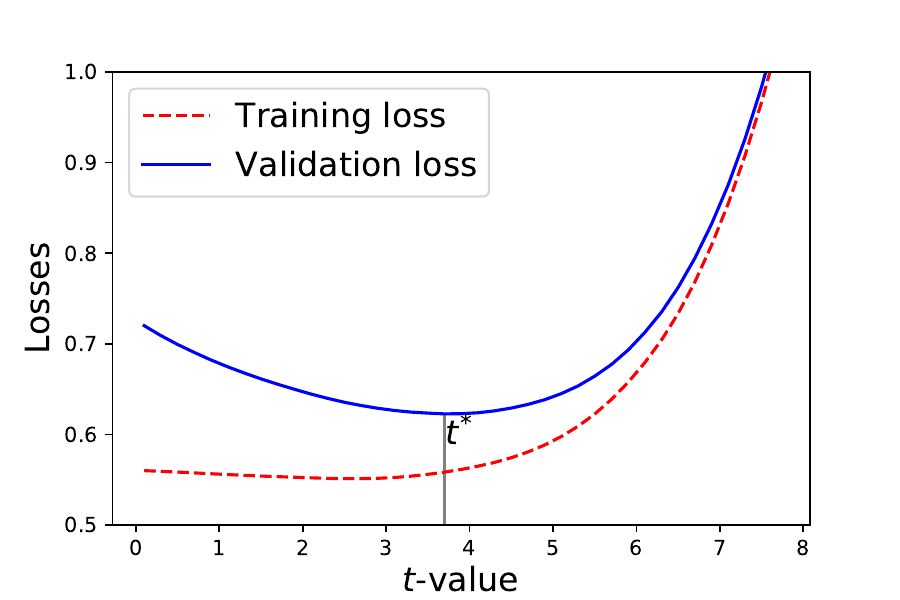}
  \captionof{figure}{Training and validation losses while moving along the steepest descent direction for MNIST (2HP).}
  \label{fig:hyper2}
\end{center}
\end{figure}
\begin{figure}[ht]
\begin{center} 
  \includegraphics[width=.40\textwidth]{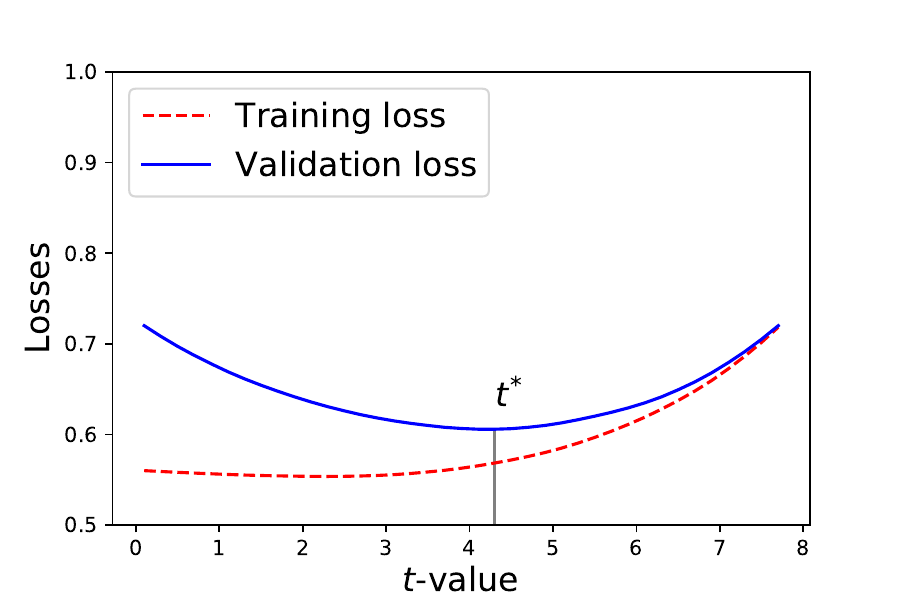}
  \captionof{figure}{Training and validation losses while moving along the steepest descent direction for MNIST (6HP).}
  \label{fig:hyper6}
\end{center}
\end{figure}
\begin{figure}[ht]
\begin{center} 
  \includegraphics[width=.40\textwidth]{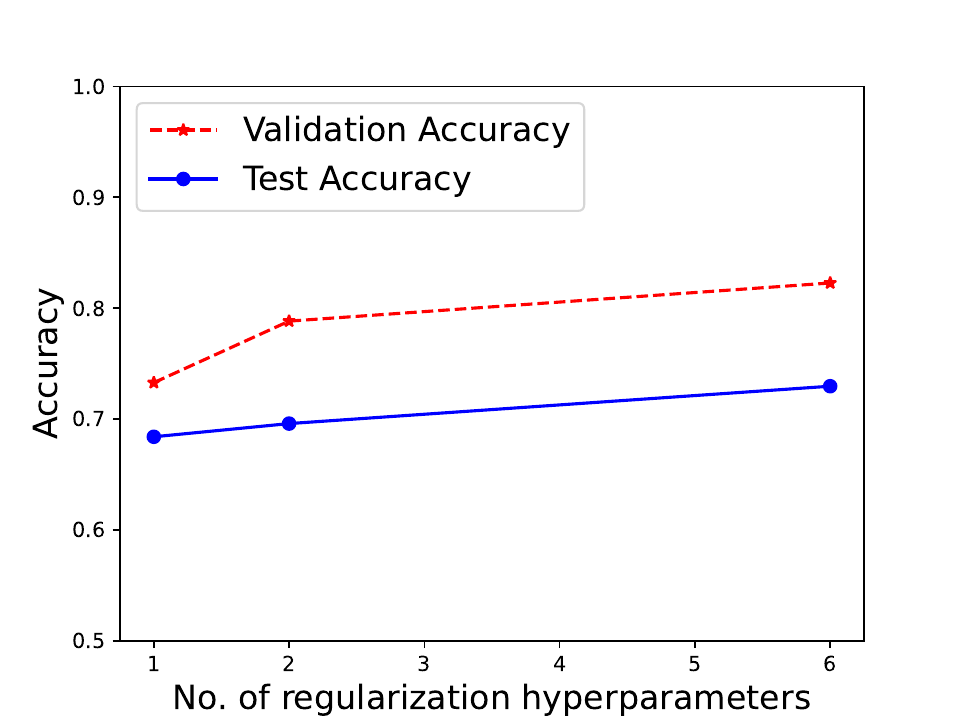}
  \captionof{figure}{Validation and test accuracy with increase in number of regularization hyperparameters. The base model with no regularization hyperparameters had the lowest test accuracy of $0.6776$.}
  \label{fig:va_te_acc}
\end{center}
\end{figure}

We present the results of fine-tuning for MNIST (1HP), MNIST (2HP) and MNIST (6HP) through Figures~\ref{fig:hyper1},~\ref{fig:hyper2} and~\ref{fig:hyper6}, respectively. For different values of $t$, we get models with (approximately) locally optimal weights on the training data. The figures show how the training and the validation accuracy for these models change as we increase $t$. Moving along the steepest descent direction leads to a better validation loss with the best being at $t^{\ast}$. The models corresponding to $t^{\ast}$ are considered to be the fine-tuned model for each of the cases. Figure~\ref{fig:va_te_acc} shows the validation and test accuracy of the models corresponding to $t^{\ast}$ for all the three cases. All the models have a better test performance than the base model with no regularization hyperparameters for which the testing accuracy was $0.6776$. Interestingly, the models improve with increase in number of regularization hyperparameters; however, this is expected to continue only until overfitting does not happen on the validation dataset. With too many regularization hyperparameters there can be overfitting on the validation dataset leading to poor performance on the test dataset.

\begin{figure*}[!ht]
\begin{center} 
  \includegraphics[width=.60\textwidth]{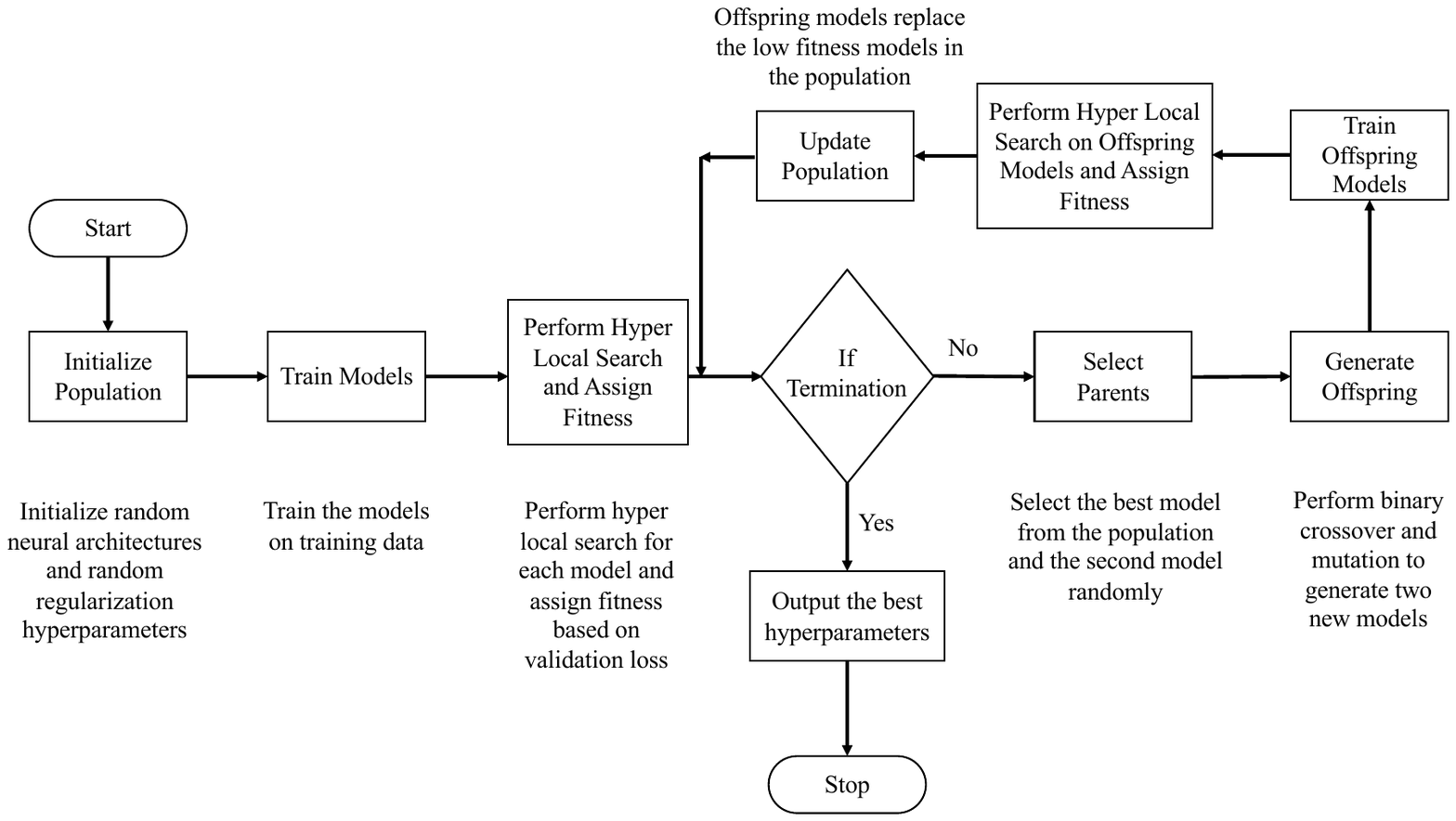}
  \vspace{-8mm}
  \captionof{figure}{Flowchart for the steady state micro-GA enhanced with linear program-based hyper local search for hyperparameter optimization.}
  \label{fig:flowchart}
\end{center}
\begin{center} 
\vspace{-10mm}
  \includegraphics[width=.75\textwidth]{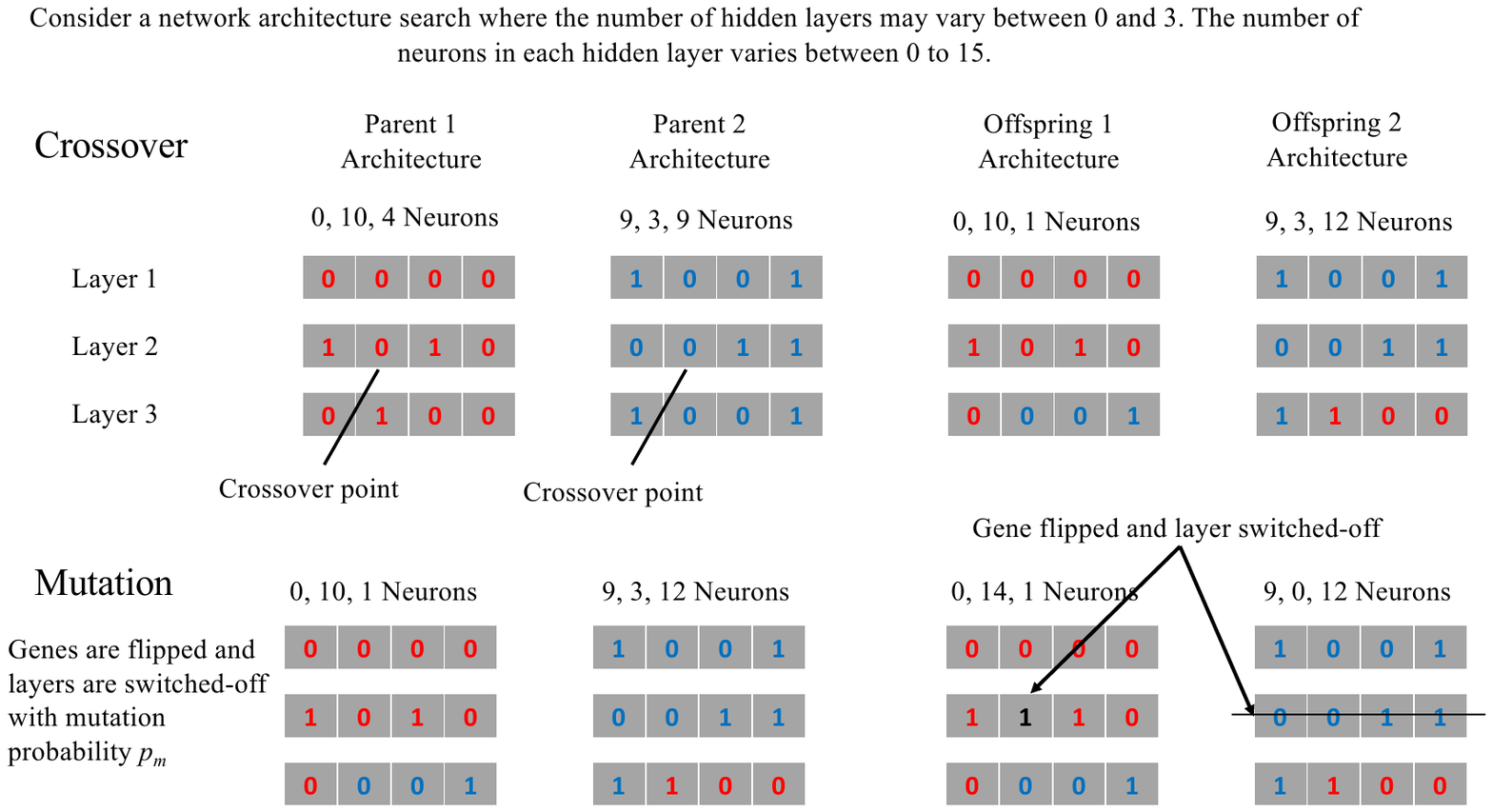}
  \vspace{-8mm}
  \captionof{figure}{Implementation of crossover and mutation on neural architectures.}
  \label{fig:operators}
\end{center}
\end{figure*}

\section{Micro Genetic Algorithm with Fine-tuning}\label{propmethod}
In this section we propose a steady-state micro-GA that utilizes the linear program-based hyper local search (fine-tuning) for the purpose of hyperparameter optimization. We will consider both discrete and continuous hyperparameters in our micro-GA. A steady state\footnote{In a generational genetic algorithm, the entire population is updated when moving from one generation to the other generation, while in a steady state genetic algorithm only a few members of the population are updated.} micro-GA starts with a small population and updates only a few solutions in each generation. Given the computational cost involved in hyperparameter optimization a steady state micro-GA is a viable option for hyperparameter search, which further gets enhanced with a linear program-based hyper local search on continuous hyperparameters. The flowchart for the micro-GA is provided in Figure~\ref{fig:flowchart}. In this paper, we consider three hyperparameters for our experiments, namely, number of hidden layers, number of neurons in each hidden layer, and regularization hyperparameters. 
On neural architecture hyperparameters (discrete) we use a binary crossover and mutation operator as shown in Figure~\ref{fig:operators}, while for regularization hyperparameters (continuous) we use the simulated binary crossover (SBX) and polynomial mutation operators \cite{deb1995simulated}. The algorithm terminates based on maximum number of generations.
The parameter settings for the micro-GA are as follows:
\begin{enumerate}
    \item Crossover probability: $p_c=0.9$
    \item Mutation probability: $p_m=0.1$
    \item Population size: 10
    \item Maximum generations: 15
    \item Offspring produced in each generation: 2
\end{enumerate}
The micro-GA can be run with or without hyper local search that we will explicitly specify while presenting the results.


\section{Results}\label{sec:experiments}
 In this section, we provide the results of micro-GA (with and without hyper local search) on two datasets. We also provide results for grid search and random search, with and without hyper local search. The objective is not to compare the performance of the genetic algorithm against naive techniques like grid search and random search, but to demonstrate that the linear program based-hyper local search proposed in the paper provides benefits in all the approaches where it is incorporated. The two datasets considered in the paper are MNIST \cite{lecun2010mnist} and CIFAR \cite{krizhevsky2009learning} with both involving a multi-class classification problem to be solved. We work with the multi-layer perceptron architecture with the following settings throughout the paper.

\begin{enumerate}
    \item Hidden layers: 0-3 (hyperparameter)
    \item Number of neurons in each layer: 0-15 (hyperparameter)
    \item L$_2$ Regularization: 1-dimensional where all weights are penalized in a single term
    \item Activation functions: ReLU in hidden layers and Softmax in output layer
    \item Optimizer: Adam
\end{enumerate}

\begin{figure*}
\centering
\begin{minipage}[t]{.49\textwidth}
\begin{center}
  \includegraphics[scale =0.48]{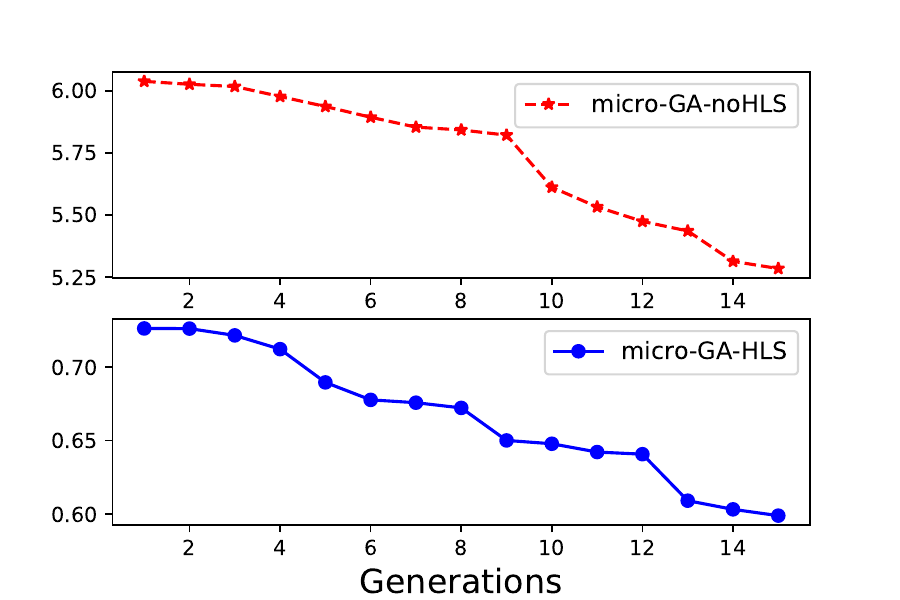}
  \captionof{figure}{Validation loss performance of micro-GA with hyper local search (micro-GA-HLS) and without hyper local search (micro-GA-noHLS) over 15 generations on MNIST dataset.}
  \label{fig:mnist-gens}
\end{center}
\end{minipage}\hfill
\begin{minipage}[t]{.49\textwidth}
\begin{center}
  \includegraphics[scale =0.48]{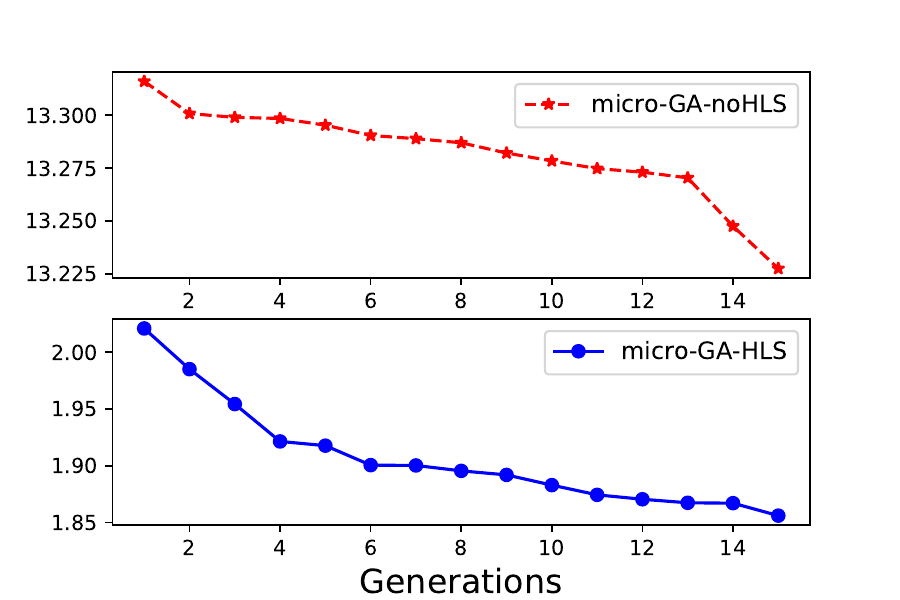}
  \captionof{figure}{Validation loss performance of micro-GA with hyper local search (micro-GA-HLS) and without hyper local search (micro-GA-noHLS) over 15 generations on CIFAR-10 dataset.}
  \label{fig:cifar-gens}
\end{center}
\end{minipage}
\end{figure*}

\subsection{MNIST Dataset}
In the original MNIST dataset, there are 60,000 training data points and 10,000 testing data points with 10 classes consisting of handwritten digits. The digits are $28 \times 28$ pixel gray-scale images. From the original training dataset we randomly sample 5000 data points that we use for training, and 2500 data points that we use for validation. The entire 10000 data points from the original testing dataset are used as testing data points. We report results for three approaches in this section, i.e., grid search, random search and genetic algorithm. For all the three approaches we report the results with and without linear program-based hyper local search. 
Number of models searched using grid search, random search and genetic algorithm are restricted to 40 in number. 
Table~\ref{tab:mnist} provides detailed results in terms of accuracy for validation and testing for various models. It is quite clear that the results with hyper local search are better in all cases.
Figure~\ref{fig:mnist-gens} shows the convergence of the micro-GA over 15 generations for a run with and without hyper local search. The losses in the case of hyper local search are much lower than the losses in the case of no hyper local search right from the start of the algorithm.

\begin{table}[]
\begin{center}
\caption{Validation and test accuracy from 40 samples of grid search, random search and micro-GA for MNIST dataset with 1 hyperparameter.}
\begin{tabular}{|l|ll|ll|}
\hline
{\color[HTML]{000000} MNIST}         & \multicolumn{2}{c|}{{\color[HTML]{000000} Without hyper local search}} & \multicolumn{2}{c|}{{\color[HTML]{000000} With hyper local search}}    \\ \hline
{\color[HTML]{000000} }              & {\color[HTML]{000000} Va. Acc.} & {\color[HTML]{000000} Te. Acc} & {\color[HTML]{000000} Va. Acc.} & {\color[HTML]{000000} Te. Acc} \\ \hline
{\color[HTML]{000000} Grid search}   & {\color[HTML]{000000} 0.7288}   & {\color[HTML]{000000} 0.7128}  & {\color[HTML]{000000} 0.7356}   & {\color[HTML]{000000} 0.7250}   \\ \hline
{\color[HTML]{000000} Random search} & {\color[HTML]{000000} 0.7212}   & {\color[HTML]{000000} 0.7142}  & {\color[HTML]{000000} 0.7652}   & {\color[HTML]{000000} 0.7626}  \\ \hline
{\color[HTML]{000000} micro-GA}      & {\color[HTML]{000000} 0.7368}   & {\color[HTML]{000000} 0.7361}  & {\color[HTML]{000000} 0.7941}   & {\color[HTML]{000000} 0.7788}  \\ \hline
\end{tabular}\label{tab:mnist}
\end{center}
\end{table}

\subsection{CIFAR-10 Dataset}
CIFAR-10 dataset has been used in our study that consists of 50,000 data points in training dataset and 10,000 data points in testing dataset with 10 classes. Each data point is a $32 \times 32$ pixel coloured image of an object. For CIFAR-10 we randomly sample 5000 data points for training, 2500 data points for validation and 10000 data points for testing from the original datasets. The results are presented in a similar manner as before for grid search, random search and genetic algorithm. 
Table~\ref{tab:cifar} clearly demonstrates the benefit of hyper local search for all the cases once again. Figure~\ref{fig:cifar-gens} shows the convergence of the micro-GA over generations for a run with and without hyper local search. Clearly, right from the start of the algorithm, the losses in the case of hyper local search are significantly lower than the losses in the case of no hyper local search.

\begin{table}[]
\begin{center}
\caption{Validation and test accuracy from 40 samples of grid search, random search and micro-GA for CIFAR-10 dataset with 1 hyperparameter.}
\begin{tabular}{|l|ll|ll|}
\hline
{\color[HTML]{000000} CIFAR-10}      & \multicolumn{2}{c|}{{\color[HTML]{000000} Without hyper local search}} & \multicolumn{2}{c|}{{\color[HTML]{000000} With hyper local search}}    \\ \hline
{\color[HTML]{000000} }              & {\color[HTML]{000000} Va. Acc.} & {\color[HTML]{000000} Te. Acc} & {\color[HTML]{000000} Va. Acc.} & {\color[HTML]{000000} Te. Acc} \\ \hline
{\color[HTML]{000000} Grid search}   & {\color[HTML]{212121} 0.1768}   & {\color[HTML]{212121} 0.1722}  & {\color[HTML]{212121} 0.2272}   & {\color[HTML]{212121} 0.2204}  \\ \hline
{\color[HTML]{000000} Random search} & {\color[HTML]{212121} 0.1872}   & {\color[HTML]{212121} 0.1819}  & {\color[HTML]{212121} 0.2432}   & {\color[HTML]{212121} 0.2415}  \\ \hline
{\color[HTML]{000000} micro-GA}      & {\color[HTML]{212121} 0.2496}   & {\color[HTML]{212121} 0.2511}  & {\color[HTML]{212121} 0.2781}   & {\color[HTML]{212121} 0.2701}  \\ \hline
\end{tabular}\label{tab:cifar}
\end{center}
\end{table}

\section{Conclusions}\label{sec:conclusions}
In our work, we have proposed a linear program-based approach that can be used to fine-tune any machine learning model by searching for better continuous hyperparameters in the vicinity of the hyperparameters chosen by the user. We formulated the hyperparameter optimization problem as a bilevel program and then showed how the gradient of the bilevel program can be used for fine-tuning the continuous hyperparameters and the model parameters. We first demonstrated the working of this principle on individual models and then incorporated this idea as hyper local search in a stead-state micro-GA. Our results show that when the proposed idea is incorporated in naive techniques like grid search and random search, or in a genetic algorithm, it benefits by producing models that perform better on validation and test data. We evaluated the idea on two datasets, MNIST and CIFAR-10, and the results obtained from all the runs are very promising. We believe that this is a fundamental contribution as the approach can be incorporated in any hyperparameter optimization algorithm. However, the approach requires Hessian computation that can make it prohibitive for large problems. As an extension, we aim to reduce the computational cost arising from Hessian computations by using approximate-Hessian techniques.


\bibliographystyle{abbrv}
\bibliography{references}

\end{document}